\documentclass[11pt,a4paper]{article}

\usepackage[utf8]{inputenc}
\usepackage[T1]{fontenc}
\usepackage{times}
\usepackage[margin=1in]{geometry}
\usepackage{amsmath,amssymb,amsthm}
\usepackage{booktabs}
\usepackage{graphicx}
\usepackage{subcaption}
\usepackage{xcolor}
\usepackage{listings}
\usepackage{enumitem}
\usepackage{hyperref}
\usepackage{url}
\usepackage{natbib}
\usepackage{authblk}
\usepackage{float}
\usepackage{multirow}
\usepackage{tabularx}

\hypersetup{
  colorlinks=true,
  linkcolor=blue!70!black,
  citecolor=blue!70!black,
  urlcolor=blue!70!black,
}

\lstdefinestyle{protocol}{
  basicstyle=\rmfamily\small,
  keywordstyle=\bfseries,
  commentstyle=\itshape,
  stringstyle=\itshape,
  breaklines=true,
  frame=single,
  backgroundcolor=\color{gray!5},
  columns=fullflexible,
  keepspaces=true,
  numbers=none,
  xleftmargin=3pt,
  xrightmargin=3pt,
  language=Python,
  morekeywords={def,return,for,if,in,range,True,False,None},
}

\newcommand{\dci}{\textsc{DCI}}
\newcommand{\dcicf}{\textsc{DCI-CF}}
\newcommand{\atoa}{\textsc{A2A}}
\newcommand{\mcp}{\textsc{MCP}}
\newcommand{\jamjet}{\textsc{jamjet}}

\newcommand{\eg}{\textit{e.g.}}

\newcommand{\etal}{\textit{et al.}}

\newcommand{\act}[1]{\textsc{#1}}  

\newtheorem{theorem}{Theorem}
\newtheorem{definition}{Definition}

\newtheorem{remark}{Remark}

\begin{document}

\title{\textbf{From Debate to Deliberation:\\Structured Collective Reasoning with Typed Epistemic Acts}}

\author[1]{Sunil Prakash}
\affil[1]{Indian School of Business, India \\ \textit{sunil\_prakash\_pgpmax2026@isb.edu}}

\date{}
\maketitle


\begin{abstract}
Multi-agent LLM systems increasingly tackle complex reasoning tasks, yet their interaction patterns remain limited to parallel generation with voting, unstructured debate, or rigid pipeline orchestration. None of these paradigms model \emph{deliberation}---a phased process in which differentiated participants exchange typed reasoning moves, preserve disagreements as productive tensions, admit new evidence under controlled rules, and converge on an explicit, accountable outcome. We introduce \textbf{Deliberative Collective Intelligence (\dci{})}, a framework that treats collective reasoning as a first-class computational object. \dci{} specifies (1)~a delegate model with four reasoning archetypes, (2)~a phased session model, (3)~an interaction grammar of 14 typed epistemic acts organized in a three-layer model, (4)~a shared workspace for structured collective thought, and (5)~\dcicf{}, a convergent flow algorithm with formal termination guarantees that always produces a structured decision packet---including the selected option, residual objections, minority report, and reopen conditions---even under persistent disagreement. We implement \dci{} on the \jamjet{} agent runtime using Gemini 2.5 Flash and evaluate on 45 tasks across seven domains---including software architecture, policy analysis, hidden-profile integration, late-evidence revision, risk analysis, disagreement resolution, and routine decisions as a negative control---organized around four hypotheses. We find that on non-routine tasks ($n=40$), \dci{} significantly improves over unstructured debate ($+0.95$, 95\% CI $[+0.41, +1.54]$), indicating that deliberative structure matters when multiple agents interact (H1). \dci{} excels on hidden-profile tasks requiring integration of partial perspectives (9.56, the highest score of any system on any domain) and significantly outperforms all baselines on such tasks, while failing on routine decisions (5.39), confirming strong task-dependence (H2). \dci{} produces 100\% structured decision packets and 98\% minority reports---process artifacts absent from all baselines. However, \dci{} is expensive: it consumes ${\sim}62\times$ the tokens of a single agent, and single-agent generation significantly outperforms \dci{} on overall quality ($-0.60$, CI $[-1.06, -0.15]$), indicating that structured deliberation is not justified for routine tasks (H3). Component contributions are not clearly separable at our sample size (H4). \dci{}'s contribution is not that more agents are better, but that consequential decisions---especially those requiring integration of partial information, multi-stakeholder reasoning, and explicit risk surfacing---benefit from deliberative structure when process quality and accountability matter enough to justify the cost.
\end{abstract}


\section{Introduction}
\label{sec:intro}

Large language models (LLMs) have demonstrated remarkable reasoning capabilities when applied individually, yet many consequential decisions---architectural design, policy analysis, strategic planning---benefit from multiple perspectives, structured argumentation, and explicit consideration of tradeoffs. The multi-agent paradigm promises to deliver these benefits by assembling multiple LLM instances into collaborative systems. In practice, however, current multi-agent approaches operate through surprisingly limited interaction patterns.

Prior multi-agent LLM systems mainly aggregate, debate, or orchestrate. \dci{} instead models collective reasoning as \emph{deliberation}: a phased process in which differentiated delegates exchange typed epistemic acts, preserve tensions, admit new evidence, and still close with an explicit result through bounded convergence. Deliberation, in this sense, is a distinct computational primitive---not a refinement of debate, not an optimization of voting, but a qualitatively different way of organizing multi-agent interaction.

To make this distinction precise, consider the four dominant paradigms:

\begin{itemize}[leftmargin=*,itemsep=3pt]
  \item \textbf{Ensembling} (self-consistency, best-of-$N$) generates independent reasoning paths and selects among them~\citep{wang2023selfconsistency}. There is no interaction between paths, no mutual refinement, and no record of why paths disagreed.

  \item \textbf{Debate} allows agents to argue freely over answers across rounds~\citep{du2023debate, liang2023encouraging}. Interaction exists, but it is untyped: a challenge is indistinguishable from a proposal at the protocol level. There is no phased progression, no structured workspace, and no guarantee that disagreement is preserved rather than flattened.

  \item \textbf{Orchestration} (AutoGen, MetaGPT, CrewAI) chains agents in workflows with role assignments~\citep{wu2023autogen, hong2023metagpt, crewai2024}. The focus is task decomposition and handoff, not reasoning interaction. A MetaGPT pipeline does not distinguish between a proposal and a challenge, and it does not preserve dissent.

  \item \textbf{Voting} collects independent judgments and selects by majority or judge~\citep{irving2018ai}. It aggregates preferences but does not transform them through engagement.
\end{itemize}

\noindent None of these paradigms produce what we term \emph{deliberated intelligence}: decisions that emerge from structured examination where assumptions are surfaced, dissent is preserved, reasoning is typed and traceable, and the process guarantees a bounded, explicit outcome. \dci{} fills this gap by treating deliberation as a first-class computational object---a session with typed acts, phased progression, a shared workspace, tension tracking, and a convergent algorithm that always terminates with a structured decision packet.

\paragraph{Contributions.} This paper makes the following contributions:
\begin{enumerate}[leftmargin=*,itemsep=2pt]
  \item We \textbf{define deliberative collective intelligence} as a distinct interaction paradigm for multi-agent LLM systems, distinguishing it from ensembling, debate, orchestration, and voting.
  \item We introduce \textbf{\dci{}}, a session-based framework with differentiated delegate archetypes, a phased session model, an interaction grammar of 14 typed epistemic acts, explicit tension tracking through a shared workspace, and structured decision packets.
  \item We propose \textbf{\dcicf{}}, a convergent deliberation algorithm that preserves epistemic openness while guaranteeing bounded procedural closure---every session terminates with a decision packet containing the selected option, residual objections, minority report, and reopen conditions.
  \item We present an \textbf{empirical evaluation} on 45 tasks across seven domains organized around four specific hypotheses, showing that structured deliberation significantly improves over unstructured debate on non-routine tasks and excels on hidden-profile perspective-integration tasks, while exposing substantial efficiency tradeoffs against simpler alternatives and confirming task-dependence through a routine negative control.
\end{enumerate}

\noindent Figure~\ref{fig:worked-example} provides a compact end-to-end walkthrough of a \dci{} session, illustrating how typed acts, tension preservation, and the decision packet work in practice.

\subsection{What is New in DCI?}
\label{sec:novelty}

\dci{}'s contribution is not any single component but their combination into a coherent deliberation protocol. Five elements, taken together, distinguish \dci{} from all prior multi-agent LLM systems:

\begin{enumerate}[leftmargin=*,itemsep=3pt]
  \item \textbf{Typed epistemic interaction.} Agents exchange structured reasoning moves---\act{propose}, \act{challenge}, \act{bridge}, \act{synthesize}---not undifferentiated text. The protocol distinguishes a challenge from a proposal at the structural level, enabling enforceable discourse rules and analyzable interaction patterns.

  \item \textbf{Session-based deliberation.} Collective thinking is organized into phases (arrival, independent thought, mutual engagement, collective shaping, closure), not just rounds. Phases create a deliberate arc from divergence through engagement to convergence, preventing premature consensus.

  \item \textbf{Tensions as first-class objects.} Disagreements are preserved in the shared workspace, not flattened by majority rule or lost in transcript accumulation. Tensions carry structure: which positions conflict, what evidence supports each side, and what would resolve the disagreement.

  \item \textbf{Bounded openness.} New evidence and hypotheses can enter the deliberation, but under controlled admission rules (materiality, distinctness, evidence linkage, and a cutoff round). This prevents both premature closure and endless expansion.

  \item \textbf{Guaranteed procedural convergence.} Every \dci{} session terminates with a structured \emph{decision packet}---the selected option, residual objections, a minority report preserving dissent, and explicit reopen conditions---even under persistent disagreement. The convergence guarantee is procedural, not epistemic: \dci{} guarantees a fair and bounded process, not truth or optimality.
\end{enumerate}

\noindent \dci{}'s novelty lies in combining typed epistemic interaction, session-based deliberation, explicit tension preservation, bounded openness to new hypotheses, and guaranteed procedural closure into a single protocol for collective reasoning. No prior multi-agent LLM system integrates all five.


\section{Background and Related Work}
\label{sec:related}

\dci{} draws on multi-agent LLM systems, ensemble reasoning, and deliberative governance theory. We review each strand, positioning \dci{} as a synthesis that addresses gaps left by each individually.

\subsection{Multi-Agent Debate}

Du~\etal{}~\citep{du2023debate} demonstrated that multi-agent debate---where LLM instances argue over answers across multiple rounds---improves factual accuracy and mathematical reasoning over single-agent baselines. Liang~\etal{}~\citep{liang2023encouraging} further showed that encouraging divergent thinking in multi-agent debate reduces sycophantic convergence. Irving~\etal{}~\citep{irving2018ai} proposed AI safety via debate for scalable oversight.

These approaches demonstrate the value of multi-agent interaction but lack formal structure: agents communicate through free-form text, there is no typed grammar constraining interaction patterns, no phased progression from exploration to convergence, and no guarantee of termination with a structured outcome.

\subsection{Multi-Agent Frameworks}

AutoGen~\citep{wu2023autogen} provides a group-chat abstraction with a manager agent mediating turns. CAMEL~\citep{li2023camel} uses role-playing for cooperative interaction. MetaGPT~\citep{hong2023metagpt} assigns software engineering roles and coordinates through structured outputs. CrewAI~\citep{crewai2024} defines tasks and roles for multi-agent workflows.

These frameworks advance multi-agent coordination infrastructure, but they focus on \emph{task decomposition and orchestration} rather than \emph{deliberation}. An AutoGen group chat does not distinguish between a proposal and a challenge. A MetaGPT pipeline does not preserve dissent or surface hidden assumptions. \dci{} addresses a complementary need: principled structure for the \emph{reasoning interaction} itself.

\subsection{Ensemble Methods and Self-Consistency}

Wang~\etal{}~\citep{wang2023selfconsistency} introduced self-consistency, which samples multiple reasoning paths and selects the most common answer. Chen~\etal{}~\citep{chen2023frugalgpt} and Ong~\etal{}~\citep{ong2024routerllm} explored model routing and cascading. These ensemble approaches improve accuracy through diversity but operate \emph{independently}---there is no interaction between paths, no mutual refinement, and no examination of why paths disagree. \dci{} generates diversity through differentiated delegates and then \emph{uses disagreement productively} through structured challenge and synthesis.

\subsection{Deliberative Democracy and Speech Act Theory}

Habermas's theory of communicative action~\citep{habermas1984theory} argues that legitimate collective decisions emerge from discourse governed by procedural norms. Fishkin's deliberative polling~\citep{fishkin2018democracy} demonstrates empirically that structured deliberation produces more informed collective judgments than unstructured discussion or voting.

Speech act theory~\citep{austin1962speech, searle1969speech} provides the linguistic foundation for typed interaction moves: utterances are \emph{acts}, not merely information transfer, and the discourse structure should reflect the type of act being performed. \dci{} adapts these principles for LLM agents, accounting for their specific characteristics: high fluency but bounded reasoning, tendency toward sycophantic agreement, and differentiability through system prompts.

\subsection{Social Choice and Arrow's Impossibility}
\label{sec:social-choice}

Arrow's impossibility theorem~\citep{arrow1951social} establishes that no rank-order voting system can satisfy a small set of fairness criteria simultaneously. \dcicf{}'s convergence mechanism acknowledges this: rather than claiming optimal aggregation, it guarantees \emph{termination with transparency}. When a decision is forced rather than emergent, the minority report and reopen conditions make the procedural nature of the outcome explicit.

\subsection{Positioning}
\label{sec:gap-summary}

The preceding review reveals that current multi-agent LLM systems lack: (1)~a typed interaction grammar that distinguishes reasoning moves at the protocol level; (2)~a termination guarantee with transparent forced-decision procedures; (3)~first-class dissent preservation; (4)~a structured workspace rather than flat transcripts; and (5)~a principled session model that applies phased deliberation to LLM collectives. \dci{} addresses all five gaps within a single framework.


\section{The \dci{} Framework}
\label{sec:framework}

\dci{} models multi-agent reasoning as governed discourse among a small council of differentiated delegates. The framework comprises four components: a delegate model, a session model, an interaction grammar, and a shared workspace.

\subsection{Delegate Model}
\label{sec:delegates}

Without differentiated delegates, collective intelligence collapses into duplicated intelligence. Each delegate carries an identity, a reasoning style, a known limitation, a current perspective (which evolves during deliberation), and a willingness to revise.

\dci{} defines four core \textbf{archetypes} that provide complementary cognitive functions:

\begin{itemize}[leftmargin=*,itemsep=2pt]
  \item \textbf{Framer} ($\delta_F$): Defines the real problem. Clarifies ambiguity, identifies hidden dimensions, decomposes mixed issues, and determines what questions actually need answering.
  \item \textbf{Explorer} ($\delta_E$): Generates novel possibilities. Proposes unconventional paths, fresh structures, analogies, and generative expansions. Opens the solution space before the group narrows it.
  \item \textbf{Challenger} ($\delta_C$): Pressure-tests everything. Searches for hidden assumptions, weak logic, risks, blind spots, and overconfidence.
  \item \textbf{Integrator} ($\delta_I$): Combines the group's thinking into coherent direction. Identifies common patterns, synthesizes positions, manages session coherence, and builds the emerging center of gravity.
\end{itemize}

Formally, a delegate $\delta_i$ maintains local state $\sigma_i = \langle v_i, c_i, Q_i, R_i, H_i \rangle$ where $v_i$ is the current view, $c_i \in [0,1]$ is confidence, $Q_i$ is the set of open questions, $R_i$ is the set of active concerns, and $H_i$ is the history of position shifts. This state evolves through interaction: a delegate that receives a strong challenge and updates its position records the shift in $H_i$ and adjusts $c_i$ accordingly.

Archetypes constrain \emph{tendency}, not \emph{capability}: a Challenger can still propose, and an Explorer can still challenge. The archetype biases the distribution of interaction acts, implemented through differentiated system prompts.

\subsection{Session Model}
\label{sec:sessions}

A \dci{} session $\mathcal{S}$ is a bounded collaborative thinking event in which a small council of delegates works on one problem. The session unfolds through five phases:

\begin{enumerate}[leftmargin=*,itemsep=2pt]
  \item \textbf{Arrival} ($\phi_1$): The session grounds itself. The group identifies the central question, notes scope and boundaries, names uncertainties, and establishes an exploratory tone. Output: shared problem statement.
  \item \textbf{Independent First Thought} ($\phi_2$): Each delegate contributes its initial view \emph{before being shaped by others}: how it sees the problem, what it thinks matters most, what it suspects is hidden, one possible direction. This preserves diversity before social influence.
  \item \textbf{Mutual Engagement} ($\phi_3$): The heart of the session. Delegates respond to one another through typed interaction acts: extend, question, challenge, bridge, clarify, reframe, deepen. The goal is \emph{movement}---improved ideas, clearer tensions, and genuine synthesis opportunities.
  \item \textbf{Collective Shaping} ($\phi_4$): The group turns discourse into shared structure: recurring themes, strongest ideas, most important tensions, what seems central, what can be discarded. Output: common ground, live tensions, 2--4 candidate paths, and a preferred direction.
  \item \textbf{Closure} ($\phi_5$): The session closes with intellectual honesty and practical usefulness: current synthesis, remaining uncertainty, key tensions, action suggestions, and carry-forward memory.
\end{enumerate}

The phase structure creates a deliberate arc from \emph{divergence} (phases 1--2) through \emph{engagement} (phase 3) to \emph{convergence} (phases 4--5), preventing premature consensus.

\subsection{Interaction Grammar}
\label{sec:grammar}

In \dci{}, a delegate does not merely ``send a message.'' It performs one or more \emph{typed epistemic acts}---meaningful moves in collective reasoning. The grammar operates at three layers:

\begin{itemize}[leftmargin=*,itemsep=2pt]
  \item \textbf{Layer 1---Speech Mode} $\mu$: The stance of the move: exploratory, analytical, critical, integrative, reflective, or decisional.
  \item \textbf{Layer 2---Interaction Act} $\alpha$: The core move from a vocabulary of 14 typed acts (Table~\ref{tab:acts}).
  \item \textbf{Layer 3---Intent} $\iota$: The specific purpose: test assumption, open option, resolve ambiguity, connect ideas, support convergence.
\end{itemize}

\begin{table}[t]
\centering
\caption{The 14 core epistemic acts in \dci{}'s interaction grammar, organized by family.}
\label{tab:acts}
\small
\begin{tabular}{@{}clll@{}}
\toprule
\textbf{\#} & \textbf{Act} & \textbf{Family} & \textbf{Description} \\
\midrule
1  & \act{frame}      & Orienting    & Define how to view the problem \\
2  & \act{propose}    & Generative   & Put forward a candidate idea or path \\
3  & \act{clarify}    & Orienting    & Remove ambiguity or distinguish concepts \\
4  & \act{ask}        & Critical     & Open a useful question \\
5  & \act{challenge}  & Critical     & Test weakness, assumption, or consequence \\
6  & \act{extend}     & Generative   & Build on another idea \\
7  & \act{reframe}    & Orienting    & Shift the level or angle of understanding \\
8  & \act{bridge}     & Integrative  & Connect two ideas or positions \\
9  & \act{synthesize} & Integrative  & Summarize where the group seems to be \\
10 & \act{ground}     & Epistemic    & Anchor a point in evidence or constraint \\
11 & \act{update}     & Epistemic    & Revise one's own position \\
12 & \act{recommend}  & Decisional   & Suggest a next direction or action \\
13 & \act{spawn}      & Generative   & Propose a sub-session for a sub-problem \\
14 & \act{recall}     & Integrative  & Incorporate a sub-session's result \\
\bottomrule
\end{tabular}
\end{table}

A complete interaction move is a triple $m = (\mu, \alpha, \iota)$ targeting either the problem, another delegate's contribution, or the shared workspace. For example, $m = (\text{critical},\allowbreak \text{\act{challenge}},\allowbreak \text{test hidden assumption})$ targeting $m_{31}$ represents a Challenger pressure-testing a specific proposal.

The grammar distinguishes \textbf{soft moves} (exploratory, tentative---appropriate early in deliberation) from \textbf{hard moves} (decisive, committal---appropriate later), creating natural progression from open exploration to convergent decision-making.

Certain acts naturally invite specific responses, forming a \textbf{response grammar}: a \act{challenge} invites defend, refine, update, or concede; a \act{synthesize} invites affirm, sharpen, surface omission, or recommend.

\subsubsection{Design Rationale for the 14 Acts}
\label{sec:act-rationale}

The 14 acts are organized into six families, each serving a distinct cognitive function in collective reasoning. The design principle is coverage: removing any family creates a specific failure mode in the deliberation process.

\begin{itemize}[leftmargin=*,itemsep=2pt]
  \item \textbf{Orienting acts} (\act{frame}, \act{clarify}, \act{reframe}) establish and refine the problem definition. Without them, the group risks shallow or misdirected problem definitions---solving the wrong problem with high confidence.
  \item \textbf{Generative acts} (\act{propose}, \act{extend}, \act{spawn}) increase the search breadth of the collective. Without them, the option space is limited to whatever the first speaker suggests, producing premature lock-in.
  \item \textbf{Critical acts} (\act{ask}, \act{challenge}) surface hidden assumptions, weak logic, and overlooked risks. Without them, proposals pass unexamined, leading to false agreement where the group converges on an option whose weaknesses were never tested.
  \item \textbf{Integrative acts} (\act{bridge}, \act{synthesize}, \act{recall}) prevent fragmentation by connecting ideas and building shared understanding. Without them, the deliberation accumulates positions without combining them, producing a list of views rather than a collective judgment.
  \item \textbf{Epistemic acts} (\act{ground}, \act{update}) expose confidence levels and enable genuine revision. Without them, delegates assert positions without anchoring them in evidence or acknowledging when a challenge has changed their view, producing fake certainty.
  \item \textbf{Decisional acts} (\act{recommend}) enable closure by proposing concrete next steps. Without them, deliberation can cycle indefinitely through analysis without arriving at actionable outcomes.
\end{itemize}

Table~\ref{tab:act-families} summarizes this structure.

\begin{table}[t]
\centering
\caption{Act families, their cognitive functions, and the deliberation failure that occurs when the family is absent.}
\label{tab:act-families}
\small
\begin{tabular}{@{}lll@{}}
\toprule
\textbf{Act Family} & \textbf{Cognitive Function} & \textbf{Failure if Absent} \\
\midrule
Orienting   & Problem framing       & Shallow or misdirected problem definition \\
Generative  & Search breadth        & Limited option space, premature lock-in \\
Critical    & Assumption testing    & False agreement, unexamined proposals \\
Integrative & Synthesis             & Fragmentation, list of views without judgment \\
Epistemic   & Confidence \& revision & Fake certainty, no genuine belief update \\
Decisional  & Closure               & Endless discussion, no actionable outcome \\
\bottomrule
\end{tabular}
\end{table}

\subsection{Shared Workspace}
\label{sec:workspace}

A \dci{} session maintains a shared workspace $\mathcal{W}$---not merely a transcript, but a structured evolving thought-space with six sections:

\begin{enumerate}[leftmargin=*,itemsep=2pt]
  \item \textbf{Problem View}: The group's current understanding of what the problem is.
  \item \textbf{Key Frames}: Different valid perspectives on the problem.
  \item \textbf{Emerging Ideas}: Candidate concepts, approaches, or hypotheses.
  \item \textbf{Tensions}: Disagreements, ambiguities, trade-offs, and unresolved questions, captured as first-class objects.
  \item \textbf{Synthesis in Progress}: What appears to be converging.
  \item \textbf{Next Actions}: Possible follow-up tasks or decisions.
\end{enumerate}

The workspace prevents repetition, makes progress visible, and gives the session a visible center of gravity. Critically, tensions are \emph{preserved} rather than resolved prematurely---the workspace explicitly tracks open disagreements, preventing the group from faking consensus.


\section{\dcicf{}: Convergent Flow Algorithm}
\label{sec:dcicf}

\dcicf{} (Deliberative Collective Intelligence---Convergent Flow) is the algorithm that guarantees every \dci{} session terminates with a result. Its design principle is: \emph{do not force minds to agree; force the process to close fairly.}

\subsection{Algorithm Overview}

\dcicf{} proceeds through eight stages (Figure~\ref{fig:dcicf-flow}):

\begin{figure}[t]
\centering
\includegraphics[width=0.95\columnwidth]{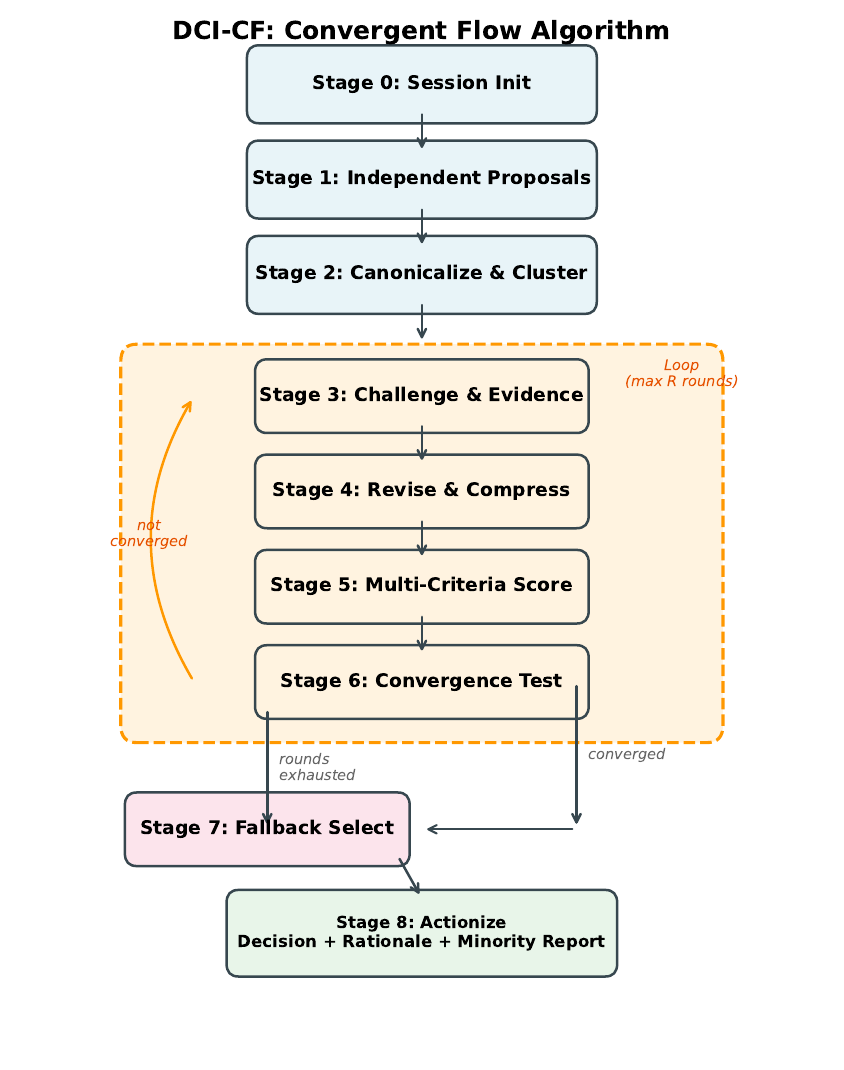}
\caption{\dcicf{} algorithm flow. Stages 3--6 form a loop bounded by \textit{max\_rounds}. If convergence fails after round exhaustion, Stage 7 provides a deterministic fallback. Every path terminates at Stage 8 with a structured decision packet.}
\label{fig:dcicf-flow}
\end{figure}

\noindent\textbf{Stage 0---Session Initialization.} Set the session envelope: problem statement $P$, delegates $\Delta = \{\delta_1, \ldots, \delta_n\}$, round budget $R_{\max}$, evaluation criteria $C = \{c_1, \ldots, c_p\}$, maximum options $K_{\max}$, finalist count $M$, convergence margin $\epsilon$, and fallback rule $\mathcal{F}$.

\noindent\textbf{Stage 1---Independent Proposal Generation.} Each delegate privately submits its framing, hypotheses, concerns, confidence level, and suggested evaluation criteria. Independent first input reduces dominance, conformity, and rhetorical capture. Output: raw hypothesis pool $\mathcal{H}$.

\noindent\textbf{Stage 2---Canonicalization and Clustering.} The algorithm deduplicates semantically equivalent ideas, splits overloaded proposals, and groups similar hypotheses into ``option families.'' Output: a finite candidate set $O = \{o_1, o_2, \ldots, o_k\}$ with $k \le K_{\max}$.

\noindent\textbf{Stage 3---Structured Challenge and Evidence.} For each option, delegates contribute support, challenge, evidence, counterexample, revision suggestion, or uncertainty note through typed epistemic acts. New hypotheses may enter only if materially distinct, plausibly superior, evidence-linked, and submitted before a cutoff round. Output: for each option, a structured record of pros, cons, assumptions, evidence, and risks.

\noindent\textbf{Stage 4---Revision and Option Compression.} Options are revised in light of criticism: refine, merge, narrow scope, split into variants, or discard if dominated. The set is compressed by removing strictly dominated options and merging compatible variants. Output: finalist set $F = \{f_1, \ldots, f_m\}$ where $m \le M$.

\noindent\textbf{Stage 5---Multi-Criteria Scoring.} Delegates evaluate finalists against the explicit criteria. Each delegate $\delta_d$ provides per-criterion scores with confidence $c_d$, evidence strength $e_d$, and rationale. The aggregate score for option $o$ is:
\begin{equation}
\label{eq:scoring}
\text{Total}(o) = \sum_{c \in C} w_c \sum_{d \in \Delta} \left[ s_{d,o,c} \cdot c_d \cdot e_d \cdot \phi_d \right]
\end{equation}
where $w_c$ is the criterion weight, $s_{d,o,c}$ is the raw score, and $\phi_d$ is a domain-fit factor reflecting the delegate's relevance to criterion $c$.

\noindent\textbf{Stage 6---Convergence Test.} The algorithm tests whether sufficient convergence has been achieved via any of: (a) \emph{score dominance}---the top option exceeds the second by margin $\epsilon$; (b) \emph{majority backing}---the top option has support above a threshold; or (c) \emph{no blocking objection}---all remaining objections are non-fatal. If none hold and rounds remain, return to Stage 3. If rounds are exhausted, proceed to Stage 7.

\noindent\textbf{Stage 7---Forced-Decision Fallback.} This guarantees termination. The algorithm applies, in sequence: (1) outranking (option winning the most pairwise comparisons), (2) minimax regret, (3) robust satisficing, (4) Integrator selection from the top-2. The output includes the chosen option, why it won \emph{procedurally}, what objections remain, and what would change the decision.

\noindent\textbf{Stage 8---Actionization and Carry-Forward.} The session closes with a structured decision packet (Definition~\ref{def:decision-packet}).

\subsection{Decision Packet}
\label{sec:decision-packet}

\begin{definition}[Decision Packet]
\label{def:decision-packet}
Every \dcicf{} session terminates with a \textbf{decision packet} $\mathcal{D}$, a structured record containing:
\begin{enumerate}[leftmargin=*,itemsep=1pt]
  \item \textbf{Selected option} with rationale and supporting evidence
  \item \textbf{Residual objections}---challenges that were raised but not resolved
  \item \textbf{Minority report}---positions held by dissenting delegates, including their reasoning and confidence levels
  \item \textbf{Next actions}---concrete follow-up steps derived from the selected option
  \item \textbf{Reopen triggers}---conditions under which the decision should be reconsidered (\eg{}, new evidence, changed assumptions, threshold events)
\end{enumerate}
\end{definition}

\noindent The decision packet is the primary output of \dci{}. It captures not only the decision but the \emph{epistemic state} of the collective at closure: what was agreed, what was contested, what was left unresolved, and what would change the outcome.

\subsection{Guarantees and Non-Guarantees}
\label{sec:guarantees}

The distinction between what \dci{} guarantees and what it does not is central to understanding its contribution.

\paragraph{DCI guarantees:}
\begin{itemize}[leftmargin=*,itemsep=1pt]
  \item \textbf{Session termination.} Every session ends in bounded time (Theorem~\ref{thm:termination}).
  \item \textbf{Bounded deliberation.} The number of rounds, options, and recursive sessions are all finitely bounded.
  \item \textbf{Explicit result packet.} Every session produces a structured decision packet (Definition~\ref{def:decision-packet}), even when natural convergence fails.
  \item \textbf{Preserved minority report.} Dissenting positions are recorded with rationale and confidence, never silently discarded.
  \item \textbf{Reopen conditions.} The decision packet specifies conditions under which the decision should be revisited.
\end{itemize}

\paragraph{DCI does \emph{not} guarantee:}
\begin{itemize}[leftmargin=*,itemsep=1pt]
  \item \textbf{Truth.} Structured deliberation cannot overcome factual deficiencies in the delegate model.
  \item \textbf{Unanimity.} Persistent disagreement is a valid outcome, preserved in the minority report.
  \item \textbf{Optimality.} The selected option is the procedural winner, not provably the best possible answer.
  \item \textbf{Superior performance on all tasks.} As our experiments confirm (Section~\ref{sec:results}), \dci{}'s coordination overhead can exceed its benefits for tasks where a single coherent generation suffices.
\end{itemize}

\noindent \dci{} does not guarantee truth or consensus. It guarantees a fair and bounded process that transforms divergent perspectives into an explicit, actionable outcome with transparent provenance.

\subsection{Convergence Theorem}
\label{sec:convergence}

\begin{theorem}[Termination]
\label{thm:termination}
For any \dci{} session with finite delegate set $|\Delta| = n$, maximum rounds $R_{\max}$, maximum options $K_{\max}$, and maximum recursion depth $D_{\max}$, \dcicf{} terminates in at most
\[
T_{\max} = R_{\max} \cdot \left(\sum_{d=0}^{D_{\max}} B_d\right)
\]
rounds, where $B_d$ is the maximum number of sessions at depth $d$.
\end{theorem}

\begin{proof}[Proof sketch]
The proof relies on six hard constraints:

\emph{(1) Finite rounds.} Each session executes at most $R_{\max}$ iterations of the Stage 3--6 loop.

\emph{(2) Finite option space.} Stage 2 caps options at $K_{\max}$; Stage 4 compresses to at most $M$ finalists. The option set is non-increasing across rounds.

\emph{(3) Hypothesis cutoff.} New hypotheses cannot enter after round $R_{\max} - 1$, ensuring the option set stabilizes.

\emph{(4) Structured scoring.} Stage 5 maps all positions to comparable numeric scores, making disagreement resolvable by comparison.

\emph{(5) Deterministic fallback.} Stage 7 provides a total ordering over finalists through a cascade of resolution methods, guaranteeing exactly one winner.

\emph{(6) Bounded recursion.} Recursive sessions are depth-limited ($D_{\max}$, default 2) and budget-carved, with a tree-wide ceiling on total rounds (default: 50).

Since every execution path either converges at Stage 6 (within $R_{\max}$ rounds) or terminates at Stage 7, and recursive spawns are bounded, the algorithm terminates. \qed
\end{proof}

\begin{remark}
Theorem~\ref{thm:termination} guarantees \emph{termination}---every \dcicf{} session produces a decision in bounded time---but makes no claim about \emph{decision quality}. Quality depends on delegate capabilities, domain knowledge, and archetype adherence. The convergence algorithm ensures the process closes fairly; whether the outcome is good depends on the substance of the deliberation.
\end{remark}

\subsection{Complexity Analysis}

\textbf{Round complexity.} Each round involves $O(n \cdot k)$ delegate-option interactions in Stage 3, $O(k^2)$ pairwise comparisons in Stage 4, and $O(n \cdot m \cdot p)$ scoring evaluations in Stage 5. Total per-session: $O(R_{\max} \cdot n \cdot k \cdot p)$.

\textbf{Token complexity.} Each delegate produces $O(L)$ tokens per interaction, where $L$ is the response length bound. Total tokens per session: $O(R_{\max} \cdot n \cdot k \cdot L)$. With default parameters ($R_{\max}=2$, $n=4$, $k=5$, $L \approx 2000$), this yields approximately 80K tokens per session---though observed mean usage (238K tokens) exceeds this estimate due to multi-stage prompts, workspace state, and context accumulation across rounds.

\textbf{Recursion overhead.} With maximum depth $D_{\max}$ and at most $B$ child sessions per level, worst-case total rounds are $R_{\max} \cdot (1 + B)^{D_{\max}}$, bounded by the tree-wide ceiling of 50 rounds.


\section{Implementation}
\label{sec:implementation}

We implement \dci{} as a workflow on the \jamjet{} agent runtime~\citep{jamjet2024}, an open-source system for composing autonomous agents with built-in support for agent-to-agent communication via \atoa{}~\citep{google2024a2a} and tool integration via \mcp{}~\citep{anthropic2024mcp}.

\subsection{Architecture}

The implementation maps directly onto \jamjet{}'s workflow model:

\begin{itemize}[leftmargin=*,itemsep=2pt]
  \item \textbf{Delegates as agents.} Each archetype is a \jamjet{} agent with a differentiated system prompt encoding its reasoning orientation and behavioral guidelines.
  \item \textbf{\dcicf{} as workflow graph.} The eight stages are expressed as a \jamjet{} workflow graph with conditional edges (Stage 6 loops back to Stage 3 or advances to Stage 7/8). Session state is maintained as workflow state.
  \item \textbf{Workspace as structured state.} The shared workspace is a structured JSON document within the workflow state, updated through typed epistemic acts. Each act modifies a specific workspace section rather than appending to a flat transcript.
  \item \textbf{Grammar enforcement.} Delegate outputs are parsed against the interaction move schema. Invalid moves are rejected and the delegate is re-prompted.
\end{itemize}

\subsection{Model Configuration}

All delegate agents use Gemini 2.5 Flash~\citep{google2024gemini} accessed through the Google Generative Language API via its OpenAI-compatible endpoint, with \textit{reasoning\_effort=``none''} to disable thinking tokens, \textit{max\_tokens=16384}, and temperature 0.7. A minimum 4-second delay between API requests enforces rate limiting. For LLM-as-judge evaluation, we use Gemini 3 Flash Preview~\citep{google2024gemini} (temperature 0.2) with structured rubrics to enable cost-effective automated assessment with a more capable model than the delegates themselves.

Delegate system prompts encode both the archetype orientation and grammar constraints. Each prompt specifies the delegate's cognitive focus, preferred act types, known limitations, and instructions for producing well-formed interaction moves with the three-layer structure (mode, act, intent). All conditions achieved 100\% task completion rate.

\subsection{Worked Example}
\label{sec:worked-example}

To illustrate how \dci{} produces structured deliberation rather than unstructured chat, Figure~\ref{fig:worked-example} traces a condensed session on an architectural design task.

\begin{figure}[t]
\centering
\fbox{\parbox{0.95\columnwidth}{\footnotesize
\textbf{Problem.} Design a high-throughput event processing pipeline handling 100K events/sec with exactly-once delivery semantics and a complete audit trail. The team has 4~engineers and a 3-month deadline.

\medskip
\textbf{Phase 1---Independent First Thought} (4 delegates, 8 moves). Selected moves:

\smallskip
\begin{tabular}{@{}p{0.12\columnwidth}p{0.14\columnwidth}p{0.66\columnwidth}@{}}
\textit{Actor} & \textit{Act} & \textit{Summary} \\[2pt]
Framer & \act{clarify} & Does ``exactly-once'' mean delivery to ingestion, or end-to-end including all downstream side effects? This distinction drives the entire architecture. \\[2pt]
Explorer & \act{propose} & Treat the event stream as an append-only immutable database (blockchain analogy): each batch cryptographically linked, providing inherent audit trail and unique event IDs for exactly-once. \\[2pt]
Challenger & \act{challenge} & Challenges the implicit assumption that exactly-once delivery \emph{and} complete audit trail are achievable at 100K~evt/s within 3~months with 4~engineers. ``Exactly-once is notoriously difficult.'' [\textit{hard move}] \\[2pt]
Integrator & \act{frame} & Frames the core tension: ambitious correctness requirements vs.\ severe resource/time constraints. Should we determine if ``at-least-once with idempotent processing'' is a more realistic goal? [\textit{hard move}]
\end{tabular}

\medskip
\textbf{Phase 2---Mutual Engagement} (4 moves). Selected moves:

\smallskip
\begin{tabular}{@{}p{0.12\columnwidth}p{0.14\columnwidth}p{0.66\columnwidth}@{}}
Framer & \act{reframe} & Distinguishes ``exactly-once \emph{delivery}'' (unachievable in distributed systems) from ``exactly-once \emph{processing}'' (idempotent consumers). Reframing to the latter simplifies the problem. [\textit{hard move}] \\[2pt]
Challenger & \act{challenge} & Challenges the immutable-database option: building a blockchain-like ledger for 100K~evt/s is a massive undertaking---risks overruns and brittleness with a 4-person team.
\end{tabular}

\medskip
\textbf{Tensions surfaced} (tracked as first-class workspace objects):
\begin{itemize}[leftmargin=1.5em,itemsep=1pt,topsep=2pt]
\item[\textbullet] Correctness vs.\ feasibility: exactly-once at 100K~evt/s vs.\ 3-month deadline / 4~engineers.
\item[\textbullet] Elegance vs.\ risk: immutable-ledger approach is conceptually clean but may be impractical.
\end{itemize}

\medskip
\textbf{Convergence.} Score dominance after round~1. Winning option: redefine delivery guarantees to ``exactly-once processing'' via Kafka + idempotent consumers (at-least-once delivery, application-level deduplication). Immutable-database option retained in minority report.

\medskip
\textbf{Decision packet} (abbreviated):
\begin{itemize}[leftmargin=1.5em,itemsep=1pt,topsep=2pt]
\item[\textbullet] \textit{Decision:} Shift to exactly-once \emph{processing} with idempotent consumers on a managed Kafka backbone.
\item[\textbullet] \textit{Minority report:} Immutable-database option remains conceptually superior for auditability but impractical within constraints.
\item[\textbullet] \textit{Reopen conditions:} (1)~requirement explicitly demands true exactly-once \emph{delivery}; (2)~throughput requirement increases 10$\times$; (3)~timeline or team size reduced further.
\end{itemize}

\smallskip
\textit{Session: 12 moves, 2 rounds, 26.5K tokens, convergence via score dominance (no forced fallback).}
}}
\caption{Condensed \dci{} session on an architectural design task. Four archetypes exchange typed epistemic acts across two phases, surfacing tensions that are preserved as first-class objects. The session converges on a pragmatic reframing while retaining the minority position and specifying reopen conditions---artifacts absent from single-agent or debate outputs.}
\label{fig:worked-example}
\end{figure}


\section{Experimental Setup}
\label{sec:experiments}

We evaluate \dci{} against four baselines across seven domains (45 tasks), with ablation conditions to probe component contributions. Rather than framing the evaluation as a broad comparison, we organize it around four specific hypotheses that the data can support or refute.

\subsection{Hypotheses}

\begin{itemize}[leftmargin=*,itemsep=3pt]
  \item \textbf{H1: Structured deliberation improves over unstructured multi-agent debate.} If the \emph{structure} of multi-agent interaction matters, \dci{} should outperform free-form debate among the same number of agents on the same tasks.
  \item \textbf{H2: \dci{} especially helps on tasks requiring perspective integration and multi-stakeholder analysis.} If deliberative structure adds value specifically through its challenge mechanisms, tension preservation, and multi-perspective reasoning, \dci{}'s advantage should be largest on hidden-profile, risk-heavy, and policy tasks---and smallest (or negative) on routine tasks.
  \item \textbf{H3: \dci{} incurs substantial coordination overhead and is not efficient for routine tasks.} Structured deliberation requires multi-stage, multi-agent interaction. If this overhead is real, \dci{} should consume substantially more tokens than simpler approaches, and single-agent generation should achieve competitive quality at a fraction of the cost.
  \item \textbf{H4: Different \dci{} components matter differently across task classes.} If the framework's components (archetypes, typed grammar, convergence algorithm) serve distinct functions, removing each should produce different effects, potentially varying by task type.
\end{itemize}

\subsection{Evaluation Domains}

We evaluate on 45 tasks across seven domains, including a negative control, to test task-dependence:

\paragraph{Domain 1: Software Architecture (10 tasks).}
Complex design tasks with multiple valid approaches, hidden tradeoffs, and no single correct answer. Tasks span distributed systems, data modeling, API design, and infrastructure choices.

\paragraph{Domain 2: Policy Analysis (10 tasks).}
Inherently multi-perspective problems requiring structured reasoning about value tradeoffs. Tasks span technology policy, organizational governance, and societal impact assessment.

\paragraph{Domain 3: Hidden-Profile (5 tasks).}
Tasks where the correct answer requires combining partial information distributed across perspectives---no single viewpoint has the full picture. These test whether deliberative structure helps integrate fragmented knowledge.

\paragraph{Domain 4: Late-Evidence (5 tasks).}
Tasks requiring revision of an initial assessment after new information arrives. These test \dci{}'s bounded-openness mechanism for admitting new evidence under controlled rules.

\paragraph{Domain 5: Risk Analysis (5 tasks).}
Tasks centered on comprehensive risk identification, where the primary evaluation dimension is surfacing hidden assumptions, second-order effects, and failure modes.

\paragraph{Domain 6: Disagreement-Heavy (5 tasks).}
Tasks with genuinely competing valid positions where reasonable experts would disagree. These test whether \dci{}'s tension preservation and minority report mechanisms add value.

\paragraph{Domain 7: Routine (5 tasks, negative control).}
Straightforward tasks with clear correct approaches that should not benefit from multi-agent deliberation. This domain tests H3's prediction that \dci{}'s overhead is not justified for simple decisions.

\subsection{Baselines}

\begin{itemize}[leftmargin=*,itemsep=2pt]
  \item \textbf{B1: Single Agent.} One LLM (same model as \dci{} delegates) given the full problem with careful-reasoning instructions and structured output format.
  \item \textbf{B2: Unstructured Debate.} Four LLMs communicating via free-form messages with no grammar, no phases, no workspace, and no \dcicf{}.
  \item \textbf{B3: Simple Voting.} Four LLMs independently produce answers; an LLM judge selects the best.
  \item \textbf{B4: Self-Consistency}~\citep{wang2023selfconsistency}. Single LLM generates multiple reasoning paths; best answer selected.
\end{itemize}

\subsection{Ablation Conditions}

\begin{itemize}[leftmargin=*,itemsep=2pt]
  \item \textbf{A1: No Archetypes.} All four delegates are generic reasoning agents with no archetype specialization. Same \dcicf{} flow.
  \item \textbf{A2: No Typed Grammar.} Delegates communicate in free-form text but still follow \dcicf{}'s staged process.
  \item \textbf{A3: No \dcicf{}.} Delegates have archetypes and grammar but deliberate freely for a fixed number of rounds with no structured convergence. Final answer extracted by an LLM summarizer.
  \item \textbf{A4: No Workspace.} Delegates have archetypes, grammar, and \dcicf{} but no shared workspace. They see the full transcript instead.
\end{itemize}

\subsection{Metrics}

\paragraph{Primary metrics.} \emph{Decision quality}: LLM-as-judge scoring on a 1--10 rubric covering completeness, reasoning quality, risk identification, tradeoff articulation, and actionability. \emph{Reasoning depth}: count of identified tradeoffs, risks, and assumptions. \emph{Perspective coverage}: count of distinct viewpoints surfaced.

\paragraph{Secondary metrics.} Total tokens consumed, wall-clock time, rounds to convergence, convergence method (natural vs.\ forced fallback), and quality-per-token ratio.

\subsection{Evaluation Protocol}

All system outputs are evaluated blind: the LLM judge does not know which system produced which output. For each task, all systems' outputs are collected and evaluated with the same rubric. We use Gemini 3 Flash Preview~\citep{google2024gemini} as the LLM judge with structured evaluation prompts that score each dimension independently before producing an aggregate score.

\subsection{Statistical Methodology}

We use paired comparisons (each task evaluated by all systems), bootstrap 95\% confidence intervals (10{,}000 resamples) for all reported metrics, Wilcoxon signed-rank tests for significance, and Holm-\v{S}id\'{a}k correction for multiple comparisons.


\section{Results}
\label{sec:results}

We present results organized by hypothesis, using the same data and tables but framing each finding as evidence for or against a specific claim.

\subsection{H1: Structured Deliberation Improves Over Unstructured Debate}

Table~\ref{tab:main-results} presents the main comparison across all seven evaluation domains. The core test of H1 is the comparison between \dci{} and unstructured debate (B2), which isolates the effect of deliberative structure while holding the number of agents constant.

\begin{table}[t]
\centering
\caption{Main results across all seven evaluation domains ($N=$ number of task evaluations per condition). Overall quality, risk identification, reasoning depth, and actionability scored 1--10 by LLM-as-judge (Gemini 3 Flash Preview). Bold indicates best in column.}
\label{tab:main-results}
\small
\begin{tabular}{@{}lccccc@{}}
\toprule
\textbf{System} & $N$ & \textbf{Overall} & \textbf{Risk ID} & \textbf{Depth} & \textbf{Actionability} \\
\midrule
\dci{} (Ours)           & 45 & 8.24 & 8.84 & 8.62 & 8.96 \\
B1: Single Agent        & 45 & \textbf{8.84} & \textbf{9.42} & \textbf{9.00} & 9.00 \\
B3: Simple Voting       & 45 & 8.78 & 8.66 & 8.67 & \textbf{9.44} \\
B4: Self-Consistency    & 45 & 8.65 & 8.37 & 8.53 & 9.32 \\
B2: Unstructured Debate & 45 & 7.75 & 8.39 & 7.72 & 8.76 \\
\bottomrule
\end{tabular}
\end{table}

\paragraph{Evidence for H1.} On the full task set ($n=45$), \dci{} scores 0.49 points higher than unstructured debate on overall quality (8.24 vs.\ 7.75). However, bootstrap 95\% confidence intervals (10{,}000 resamples) show this overall difference is \emph{not} statistically significant: $\Delta = +0.49$, 95\% CI $[-0.10, +1.12]$. The reason is informative: routine tasks (Domain 7) drag down \dci{}'s average substantially (5.39 on routine vs.\ 8.58 for debate), diluting the signal from domains where deliberative structure genuinely helps.

On \textbf{non-routine tasks} ($n=40$), the picture changes: \dci{} scores $+0.95$ over debate, 95\% CI $[+0.41, +1.54]$---\textbf{statistically significant}. This is the most direct test of whether deliberative structure matters when the task warrants it: the same number of agents, the same model, the same tasks---the only difference is the presence of typed acts, phased sessions, a shared workspace, and a convergence algorithm.

\paragraph{Context.} H1 is supported on non-routine tasks but \emph{not} on the full task set, confirming that deliberative structure helps selectively. Single-agent generation now significantly outperforms \dci{} overall ($\Delta = -0.60$, 95\% CI $[-1.06, -0.15]$). Simple voting (8.78) and self-consistency (8.65) also score higher. The data show that \dci{}'s value is task-dependent: it improves over debate on perspective-heavy and process-sensitive tasks, but its failure on routine tasks confirms that structured deliberation is not a general-purpose improvement. \dci{}'s value must be assessed domain-by-domain, which we examine next.

\subsection{H2: DCI Especially Helps on Perspective-Integration and Process-Sensitive Tasks}

Table~\ref{tab:domain-split} breaks results by domain across all seven evaluation categories. We examine whether \dci{}'s advantage concentrates on the task types where deliberative structure should matter most.

\begin{table}[t]
\centering
\caption{Overall quality scores by domain. Bold indicates best in each column. $n$ = tasks per domain.}
\label{tab:domain-split}
\small
\begin{tabular}{@{}lccccccc@{}}
\toprule
\textbf{System} & \rotatebox{70}{\textbf{Hidden-Prof.}} & \rotatebox{70}{\textbf{Late-Evid.}} & \rotatebox{70}{\textbf{Risk}} & \rotatebox{70}{\textbf{Disagree.}} & \rotatebox{70}{\textbf{Policy}} & \rotatebox{70}{\textbf{Arch.}} & \rotatebox{70}{\textbf{Routine}} \\
 & \textit{n=5} & \textit{n=5} & \textit{n=5} & \textit{n=5} & \textit{n=10} & \textit{n=10} & \textit{n=5} \\
\midrule
\dci{} (Ours)    & \textbf{9.56} & 9.24          & 8.48          & 8.15          & 8.55          & 8.13          & 5.39 \\
Single Agent     & 9.25          & \textbf{9.60} & \textbf{8.85} & \textbf{8.87} & 8.82          & 8.73          & 7.88 \\
Voting           & 9.30          & 9.26          & 8.03          & 8.68          & 8.26          & \textbf{9.19} & 8.86 \\
Self-Consistency & 9.08          & 8.80          & 8.77          & 7.83          & \textbf{8.97} & 8.22          & \textbf{8.96} \\
Unstr.\ Debate   & 9.03          & 8.45          & 7.83          & 8.24          & 8.03          & 5.78          & 8.58 \\
\bottomrule
\end{tabular}
\end{table}

\paragraph{Evidence for H2: hidden-profile tasks.} \dci{}'s strongest domain is hidden-profile tasks (9.56)---the highest score of any system on any domain. Hidden-profile tasks require integrating partial information that no single perspective possesses, and this is exactly where differentiated delegates and structured engagement should help. \dci{} significantly outperforms single-agent generation on hidden-profile tasks ($\Delta = +0.31$, 95\% CI $[+0.12, +0.49]$)---the \emph{only} domain where \dci{} beats the single agent. \dci{} also significantly outperforms self-consistency ($\Delta = +0.48$) and debate ($\Delta = +0.53$) on this domain.

\paragraph{Evidence for H2: process-sensitive tasks.} On process-sensitive tasks (architecture + policy, $n=20$), DCI--Debate is $+1.44$, 95\% CI $[+0.57, +2.43]$---statistically significant. On the standard architecture domain, debate degrades severely (5.78) while \dci{} maintains 8.13, a gap of $+2.36$.

\paragraph{Evidence for H2: routine negative control.} \dci{} scores 5.39 on routine tasks---significantly lower than \emph{every} baseline (DCI--Debate: $-3.19$, CI $[-4.25, -2.11]$). This negative control confirms strong task-dependence: \dci{}'s deliberative machinery actively \emph{harms} output quality on straightforward tasks, consistent with H2's prediction that \dci{}'s value is domain-specific.

\paragraph{Process metrics.} Table~\ref{tab:process-metrics} shows that \dci{}'s structural differentiation extends beyond quality scores. \dci{} produces decision artifacts that no baseline provides.

\begin{table}[t]
\centering
\caption{Process metrics across systems. Decision packet completeness, minority report presence, and reopen conditions presence are percentages. Explicit assumptions and risk/objection counts are means per task.}
\label{tab:process-metrics}
\small
\begin{tabular}{@{}lccccc@{}}
\toprule
\textbf{Metric} & \textbf{\dci{}} & \textbf{SA} & \textbf{Debate} & \textbf{Voting} & \textbf{SC} \\
\midrule
Decision packet     & \textbf{100\%} & 1\%  & 8\%  & 16\% & 0\% \\
Minority report     & \textbf{98\%}  & 0\%  & 0\%  & 0\%  & 0\% \\
Reopen conditions   & \textbf{100\%} & 0\%  & 0\%  & 0\%  & 0\% \\
Explicit assumptions & \textbf{3.6} & 3.3  & 0.2  & 0.0  & 0.0 \\
Risk count          & 4.3            & \textbf{10.5} & 3.5  & 4.9  & 3.4 \\
Objection count     & 3.8            & \textbf{7.3}  & 1.1  & 2.4  & 2.2 \\
\bottomrule
\end{tabular}
\end{table}

\subsection{H3: Substantial Coordination Overhead}

Table~\ref{tab:efficiency} presents efficiency metrics. The token costs test whether \dci{}'s deliberative structure comes at a substantial cost.

\begin{table}[t]
\centering
\caption{Efficiency metrics across systems. Mean tokens per task across all agents and rounds.}
\label{tab:efficiency}
\small
\begin{tabular}{@{}lrrr@{}}
\toprule
\textbf{System} & \textbf{Mean Tokens} & \textbf{Quality} & \textbf{Quality/kToken} \\
\midrule
B1: Single Agent        & 3,809    & \textbf{8.84} & \textbf{2.320} \\
B2: Unstructured Debate & 9,458    & 7.75 & 0.819 \\
B4: Self-Consistency    & 21,249   & 8.65 & 0.407 \\
B3: Simple Voting       & 31,987   & 8.78 & 0.274 \\
\dci{} (Ours)           & 237,565  & 8.24 & 0.035 \\
\bottomrule
\end{tabular}
\end{table}

\paragraph{Evidence for H3.} \dci{} consumes approximately $62\times$ the tokens of a single agent for an overall quality score that is 0.60 points \emph{lower}---a gap that is now statistically significant ($\Delta = -0.60$, 95\% CI $[-1.06, -0.15]$). In quality-per-token terms, single-agent generation dominates all conditions (2.320 quality/kToken vs.\ 0.035 for \dci{}). Even compared to unstructured debate, \dci{} uses ${\sim}25\times$ more tokens for a 0.49-point quality improvement that does not reach significance on the full task set. H3 is strongly supported.

\dci{}'s routine-task performance (5.39) provides the most direct evidence for H3: on straightforward tasks, \dci{}'s deliberative machinery actively degrades output quality. The multi-stage pipeline introduces coordination overhead, error propagation, and over-structuring on tasks where a single coherent generation suffices. \dci{}'s convergence process required a mean of 1.5 rounds with a 51\% fallback rate, indicating that natural convergence is achieved roughly half the time.

This cost is not a side effect---it is a central finding that shapes \dci{}'s applicability. \dci{} is not a general replacement for simpler approaches. It is designed for tasks where \emph{process quality} matters enough to justify the cost: decisions requiring explicit risk surfacing, preserved dissent, stakeholder accountability, and structured closure under disagreement.

\paragraph{When the cost is justified.} Three properties of \dci{}'s output are absent from cheaper alternatives:

\begin{enumerate}[leftmargin=*,itemsep=2pt]
  \item \textbf{Decision packets.} Every session produces a structured artifact: the selected option, residual objections, a minority report, and reopen conditions. \dci{} achieves 100\% decision packet completeness and 98\% minority report presence; no simpler baseline exceeds 16\% on either metric (Table~\ref{tab:process-metrics}).
  \item \textbf{Hidden-profile integration.} On tasks requiring combination of partial perspectives, \dci{} achieves 9.56---the highest score of any system on any domain---and significantly outperforms the single agent ($+0.31$, CI $[+0.12, +0.49]$).
  \item \textbf{Policy and architecture performance.} On process-sensitive tasks ($n=20$), \dci{} significantly outperforms debate ($+1.44$, CI $[+0.57, +2.43]$) while providing dissent artifacts that single-agent output cannot.
\end{enumerate}

\noindent The right question is not ``is \dci{} efficient?'' (it is not) but ``does the task require accountable, auditable deliberation?'' When it does, the structured decision packet and preserved minority positions provide value not captured by scalar quality scores.

\subsection{H4: Component Contributions}

Table~\ref{tab:ablation} presents results from three of four planned ablation conditions (A1--A3) run on 25 tasks each; A4 (No Workspace) was not completed within the experimental budget. We note upfront that while the sample size has increased from our initial experiments, the high variance in \dci{}'s performance continues to limit the strength of conclusions.

\begin{table}[t]
\centering
\caption{Ablation results ($n=25$ tasks each). Overall quality scored 1--10. Full \dci{} reference is the overall mean ($n=45$).}
\label{tab:ablation}
\small
\begin{tabular}{@{}lcc@{}}
\toprule
\textbf{Condition} & \textbf{Overall} & \textbf{$\Delta$ vs.\ Full} \\
\midrule
Full \dci{}              & 8.24 & ---   \\
A1: No Archetypes        & 8.61 $\pm$ 0.68 & +0.37 \\
A2: No Typed Grammar     & 8.32 $\pm$ 1.47 & +0.08 \\
A3: No \dcicf{}          & 8.33 $\pm$ 0.91 & +0.09 \\
\bottomrule
\end{tabular}
\end{table}

\paragraph{Evidence regarding H4.} As with our initial experiments, all three ablation conditions scored at or above the full framework. Removing archetypes ($+0.37$), typed grammar ($+0.08$), or the convergence algorithm ($+0.09$) each produced equal or higher mean scores. Bootstrap confidence intervals confirm that \emph{none} of these differences are statistically significant. The No Typed Grammar condition shows the highest variance ($\pm 1.47$), suggesting that the typed grammar acts as a variance reducer even when it does not improve the mean---removing it yields occasional high and low outliers.

H4 is not supported---the data do not reliably separate individual component contributions. We identify three non-exclusive explanations:

\begin{enumerate}[leftmargin=*,itemsep=2pt]
  \item \textbf{Moderate sample, high variance.} Even with $n=25$, the wide confidence intervals (especially for No Typed Grammar) confirm that reliable component attribution requires larger evaluation sets or lower-variance tasks.

  \item \textbf{Coordination overhead.} The full pipeline involves multiple stages of structured interaction, each adding opportunities for error propagation. When the model occasionally produces low-quality structured outputs, the multi-stage pipeline amplifies rather than corrects these failures.

  \item \textbf{Over-constraint.} The typed grammar and convergence algorithm impose structure that may over-constrain generation for some tasks. Generic agents communicating freely may produce more coherent outputs when the underlying model is already capable.
\end{enumerate}

We note a suggestive pattern: the No Typed Grammar condition's high variance ($\pm 1.47$ vs.\ $\pm 0.68$ for No Archetypes) suggests that the grammar provides consistency even when it does not improve average quality. Similarly, \dci{}'s hidden-profile dominance (9.56, beating all baselines) is consistent with the archetype differentiation providing genuine epistemic diversity on perspective-integration tasks. However, these observations are not isolated by the ablation design. Confirming any component-level attribution requires targeted experiments with larger samples.

\subsection{Reliability}

All conditions achieved 100\% task completion rate after an initial engineering fix (setting \textit{reasoning\_effort=``none''} to disable thinking tokens that caused parsing failures). This is notable for \dci{}, which involves the most complex multi-stage pipeline: across 45 tasks, every \dcicf{} session terminated with a structured decision packet, confirming Theorem~\ref{thm:termination} in practice. The convergence process required a mean of 1.5 rounds with a 51\% fallback rate---indicating that natural convergence is achieved roughly half the time, with the Stage 7 forced-decision mechanism providing reliable closure for the remainder.

\subsection{Qualitative Observations}

Inspection of session logs reveals patterns that complement the quantitative results:

\begin{itemize}[leftmargin=*,itemsep=2pt]
  \item \textbf{Hidden-profile integration.} On hidden-profile tasks, delegates consistently contributed distinct partial perspectives that were integrated during the synthesis phase---exactly the mechanism \dci{} is designed to support. The Framer's problem reframing and Integrator's synthesis were visibly productive on these tasks.

  \item \textbf{Coordination overhead in output coherence.} \dci{}'s multi-stage pipeline occasionally produced outputs where the final synthesis was less coherent than a single-agent generation. Aggregating multiple delegate perspectives sometimes introduced redundancy or unresolved tension in the final output---a pattern especially pronounced on routine tasks.

  \item \textbf{Policy domain strength.} On policy tasks, delegates naturally surfaced competing stakeholder interests, equity considerations, and implementation barriers that single agents addressed more superficially.

  \item \textbf{Routine task degradation.} On routine tasks, the deliberative machinery introduced unnecessary complexity: delegates generated artificial tensions on straightforward problems, and the convergence process produced over-structured outputs for questions with clear answers.

  \item \textbf{Outlier sensitivity.} Several low-scoring runs involved cascade failures where an early-stage malformed output propagated quality degradation through subsequent stages.
\end{itemize}

\subsection{Cross-Judge Validation}
\label{sec:cross-judge}

To verify that our evaluation is not an artifact of a single judge, we scored 20 representative outputs (4 per condition) using three independent LLM judges: Gemini 3 Flash Preview (our primary judge), GPT-4o (OpenAI), and Claude Sonnet~4 (Anthropic). Table~\ref{tab:cross-judge} reports the per-condition mean overall quality for each judge.

\begin{table}[h]
\centering
\caption{Cross-judge validation: mean overall quality (1--10) for 20 representative outputs scored by three independent LLM judges (4 outputs per condition).}
\label{tab:cross-judge}
\small
\begin{tabular}{@{}lccc@{}}
\toprule
Condition & Gemini & GPT-4o & Claude \\
\midrule
Voting            & 9.24 & 8.73 & 8.66 \\
Single Agent      & 9.12 & 8.73 & 8.50 \\
DCI               & 8.57 & 8.27 & 8.12 \\
Self-Consistency  & 8.50 & 8.75 & 7.62 \\
Unstructured Debate & 6.97 & 8.06 & 7.54 \\
\bottomrule
\end{tabular}
\end{table}

All three judges agree on the essential finding: voting and single-agent baselines score highest, \dci{} and self-consistency occupy the middle tier, and unstructured debate ranks last. Critically, the \dci{} $>$ debate ordering---the central test of H1---is confirmed by all three judges, ruling out single-judge bias as an explanation for our main result. Inter-judge agreement is substantial: Claude Sonnet~4 correlates most strongly with Gemini (Pearson $r = 0.817$, MAD $= 0.63$), while GPT-4o shows moderate correlation ($r = 0.592$, MAD $= 0.55$). Claude is the strictest judge (mean $8.09$) and Gemini the most generous ($8.48$), consistent with known calibration differences across model families, but relative condition rankings remain stable across all three.


\section{Analysis and Discussion}
\label{sec:discussion}

\subsection{When Does Structured Deliberation Help?}

\dci{} is not a general-purpose reasoning default. At ${\sim}62\times$ the token cost of a single agent for lower overall quality scores, it cannot be justified on efficiency grounds. Its value is narrow, specific, and conditional---tied to tasks where the \emph{process artifacts} matter as much as the final answer:

\begin{itemize}[leftmargin=*,itemsep=2pt]
  \item \textbf{When the decision requires accountable process.} \dci{}'s decision packet---selected option, residual objections, minority report, reopen conditions---provides an audit trail absent from all simpler approaches. \dci{} achieves 100\% decision packet completeness and 98\% minority report presence, compared to $\le$16\% for all baselines. For decisions requiring stakeholder accountability or regulatory justification, this structured output is the primary differentiator.

  \item \textbf{When the task requires integrating partial perspectives.} On hidden-profile tasks, \dci{} achieves 9.56---the highest score of any system on any domain---and is the \emph{only} condition that significantly outperforms single-agent generation ($+0.31$, CI $[+0.12, +0.49]$). Tasks where the correct answer requires combining fragmented information are \dci{}'s strongest use case.

  \item \textbf{When tasks involve competing perspectives.} On policy and architecture tasks requiring multi-stakeholder reasoning, \dci{} significantly outperforms debate ($+1.44$, CI $[+0.57, +2.43]$). The differentiated delegates and tension preservation mechanisms provide genuine value on process-sensitive tasks.
\end{itemize}

For routine reasoning, single-agent generation or simple voting is the better choice---\dci{}'s score of 5.39 on routine tasks confirms this empirically. \dci{} is designed for the subset of decisions where minority reports, residual objections, reopen conditions, and perspective integration provide value that justifies the cost---architectural reviews, policy deliberations, hidden-profile situations, strategic decisions with competing stakeholders, and other settings where auditability and process accountability matter.

\subsection{When Structured Deliberation Fails}
\label{sec:when-fails}

We identify five failure modes:

\begin{enumerate}[leftmargin=*,itemsep=2pt]
  \item \textbf{Fundamental knowledge gaps.} \dci{} cannot overcome factual deficiencies in the delegate model. When delegates lack domain knowledge, structured deliberation produces \emph{convergent but wrong} decisions. This is the most dangerous failure mode, because the procedural structure may create unwarranted confidence in factually incorrect outcomes.

  \item \textbf{Overhead on routine tasks.} For problems with clear correct answers, \dci{} adds multi-round overhead that actively degrades quality. Our routine domain (5.39) confirms this: \dci{} scores significantly lower than every baseline, including debate (8.58). The deliberative machinery introduces unnecessary complexity, error propagation, and over-structuring on tasks where a single coherent generation suffices.

  \item \textbf{Archetype role drift.} In extended sessions (3+ rounds), delegates may abandon their assigned cognitive orientation. A Challenger may begin proposing rather than pressure-testing. System prompts establish initial adherence, but sustained maintenance across many turns is not guaranteed.

  \item \textbf{Sycophantic convergence.} Despite the Challenger's mandate, LLMs exhibit a well-documented tendency toward agreement~\citep{liang2023encouraging}. Even designated Challengers may soften objections across rounds, producing apparent consensus that reflects social compliance rather than genuine deliberative agreement.

  \item \textbf{Forced fallback quality.} When natural convergence fails and \dcicf{} invokes the Stage 7 fallback, the resulting decision reflects \emph{procedural resolution} rather than genuine agreement. The minority report documents this distinction, but downstream consumers may not attend to it.
\end{enumerate}

\subsection{Limitations}

\begin{itemize}[leftmargin=*,itemsep=2pt]
  \item \textbf{LLM-as-judge evaluation.} Our primary evaluation relies on an LLM judge (Gemini 3 Flash Preview), which may have systematic biases toward coherent single-agent outputs~\citep{liang2023encouraging}. Human evaluation would complement automated scores, particularly for risk identification.
  \item \textbf{Model homogeneity.} The main experiments use a single delegate model (Gemini 2.5 Flash), limiting genuine epistemic diversity. A preliminary diverse-council experiment (2$\times$Gemini + 2$\times$GPT-4o, $n$=5) improved architectural-domain quality from 8.13 to 8.71, suggesting this limitation is remediable.
  \item \textbf{Sample size.} With 45 tasks across seven domains (5--10 per domain) and ablations on 25 tasks, statistical power has improved from our initial 20-task evaluation but remains limited for per-domain conclusions, especially in domains with only 5 tasks.
  \item \textbf{Archetype adherence.} LLMs may not perfectly maintain archetype behavior across extended sessions.
  \item \textbf{Task scope.} Our evaluation covers seven domains but remains focused on open-ended reasoning tasks. Generalization to mathematical reasoning, creative tasks, or negotiation requires further investigation.
  \item \textbf{Scale.} We evaluate with 4 delegates. Behavior with larger councils (7+) remains an open question.
  \item \textbf{Cost.} \dci{} consumes approximately $62\times$ the tokens of a single agent. The quality-per-token ratio strongly favors simpler approaches. \dci{}'s cost is justified only when specific quality dimensions are valued above raw efficiency.
\end{itemize}

\subsection{Connection to AI Safety}

\dci{}'s emphasis on \emph{preserved dissent} provides a structural mechanism against sycophantic tendencies. When a Challenger surfaces a valid objection that the majority dismisses, the minority report ensures the objection is recorded and available for review. This mirrors dissenting opinions in legal systems, which often influence future decisions despite not prevailing initially.

The formal convergence guarantee also contributes to safety: a \dci{} session cannot run indefinitely or silently fail to produce output. Every session terminates with explicit decisions, explicit uncertainty, and explicit reopen conditions.


\section{Conclusion}
\label{sec:conclusion}

The main contribution of \dci{} is not that more agents are better, but that collective reasoning benefits from explicit deliberative structure. Typed acts, visible tensions, and fair closure rules turn multi-agent interaction from parallel monologue into accountable collective judgment.

We evaluated \dci{} across 45 tasks in seven domains---software architecture, policy analysis, hidden-profile integration, late-evidence revision, risk analysis, disagreement resolution, and routine decisions---organized around four hypotheses. Four empirical conclusions emerge:

\begin{itemize}[leftmargin=*,itemsep=2pt]
  \item On non-routine tasks ($n=40$), \dci{} significantly improves over unstructured debate ($+0.95$, CI $[+0.41, +1.54]$), indicating that deliberative structure matters when the task warrants it.
  \item \dci{} excels on hidden-profile tasks (9.56, the highest score of any system on any domain) and is the only condition that significantly outperforms single-agent generation on any domain---precisely on tasks requiring integration of partial perspectives.
  \item \dci{} fails on routine tasks (5.39, significantly worse than all baselines), confirming that structured deliberation is task-dependent. This negative control validates that \dci{}'s value is genuine rather than artifactual.
  \item \dci{} is expensive (${\sim}62\times$ single-agent token cost) and produces 100\% structured decision packets with 98\% minority reports---process artifacts absent from all baselines---but single-agent generation significantly outperforms \dci{} on overall quality.
\end{itemize}

\noindent The right use case for \dci{} is not routine reasoning but consequential decisions---especially those requiring integration of partial information, multi-stakeholder accountability, and explicit risk surfacing---where accountable process matters enough to justify the cost.

\paragraph{Contributions.} In summary, this paper:
\begin{enumerate}[leftmargin=*,itemsep=2pt]
  \item \textbf{Defines deliberative collective intelligence} as a distinct interaction paradigm for multi-agent LLM systems, distinguishing it from ensembling, debate, orchestration, and voting.
  \item \textbf{Introduces \dci{}}, a session-based framework with differentiated delegates, explicit tension tracking, a shared workspace, and structured decision packets.
  \item \textbf{Proposes \dcicf{}}, a convergent deliberation algorithm that preserves epistemic openness while guaranteeing bounded procedural closure.
  \item \textbf{Shows empirically} that structured deliberation significantly improves over unstructured debate on non-routine tasks and excels on hidden-profile tasks requiring perspective integration---while honestly exposing that simpler baselines achieve higher overall quality at dramatically lower cost, and that \dci{} fails on routine tasks.
\end{enumerate}

\paragraph{Future work.} Several directions could address the limitations our evaluation reveals:

\begin{enumerate}[leftmargin=*,itemsep=2pt]
  \item \textbf{Model-diverse councils.} A preliminary experiment with heterogeneous delegates (2$\times$Gemini + 2$\times$GPT-4o, $n$=5) scored 8.71 mean quality on architecture tasks vs.\ 8.13 for the homogeneous council (+0.58). Genuine model diversity appears to provide real epistemic value beyond prompt-induced differentiation.

  \item \textbf{Iterative refinement.} Adding a post-deliberation refinement stage where the synthesized output is critiqued and revised could improve output coherence without requiring full re-deliberation.

  \item \textbf{Adaptive deliberation depth.} An adaptive system could assess task complexity and invoke structured deliberation only when expected value exceeds coordination cost, routing routine tasks to single-agent generation.

  \item \textbf{Larger-scale evaluation.} Expanding per-domain sample sizes beyond 5--10 tasks and running ablations on all domains would enable reliable per-domain component attribution and proper statistical testing of domain-specific effects.

  \item \textbf{Human evaluation.} Expert evaluation, particularly for risk identification and output structure, would complement LLM-as-judge scores and may reveal strengths that automated evaluation underweights.
\end{enumerate}



\appendix

\section{\dcicf{} Pseudocode}
\label{app:pseudocode}

\begin{lstlisting}[style=protocol,caption={\dcicf{} convergent flow algorithm.},label={lst:dcicf}]
def dci_convergent_flow(problem, delegates, criteria,
                        max_rounds=2, max_options=5,
                        finalist_count=3, margin=0.15,
                        depth=0, max_depth=2):
    # Stage 0: Session initialization
    workspace = init_workspace(problem, delegates, criteria)
    session = create_session(problem, delegates, depth, max_depth)

    # Stage 1: Independent proposal generation
    proposals = collect_independent_proposals(delegates, problem)

    # Stage 2: Canonicalization and clustering
    options = canonicalize_and_cluster(proposals, max_options)

    for round_idx in range(1, max_rounds + 1):
        # Stage 3: Structured challenge and evidence
        records = []
        for option in options:
            records.extend(
                collect_challenge_evidence(delegates, option, round_idx)
            )

        # Admit new hypotheses before cutoff
        if round_idx < max_rounds:
            new_options = admit_new_hypotheses(records)
            options = merge_options(options, new_options)

        # Stage 4: Revision and option compression
        options = revise_and_compress(options, records, max_options)

        # Stage 5: Multi-criteria scoring
        finalists = select_finalists(options, finalist_count)
        score_table = score_options(delegates, finalists, criteria)
        ranking = aggregate_scores(score_table)

        # Stage 6: Convergence test
        if has_dominant_winner(ranking, margin):
            return finalize_decision(ranking[0], finalists, records)
        if no_blocking_objection(finalists, records):
            return finalize_decision(ranking[0], finalists, records)

    # Stage 7: Forced-decision fallback
    winner = fallback_select(finalists, "outranking_then_minimax")
    return finalize_decision(winner, finalists, records, forced=True)

def finalize_decision(winner, finalists, records, forced=False):
    # Stage 8: Actionization and carry-forward
    return {
        "decision": winner,
        "rationale": build_rationale(winner),
        "minority_report": build_minority_report(finalists, records),
        "action_plan": derive_actions(winner),
        "assumptions": extract_assumptions(winner, records),
        "risks": extract_risks(winner, records),
        "reopen_conditions": derive_reopen_conditions(records),
        "confidence": compute_confidence(winner),
        "forced_fallback": forced
    }
\end{lstlisting}

\section{Convergence Proof}
\label{app:proof}

\begin{theorem}[Termination of \dcicf{}]
For any \dci{} session with parameters
$(n, R_{\max}, K_{\max}, M, D_{\max})$,
the \dcicf{} algorithm terminates in finite time.
\end{theorem}

\begin{proof}
We show that every execution path through \dcicf{} reaches Stage 8 in bounded steps.

\emph{Claim 1: The Stage 3--6 loop terminates.} The loop counter $r$ increments from 1 to $R_{\max}$ on each iteration. At $r = R_{\max}$, if neither convergence condition (score dominance or no blocking objection) holds, the algorithm exits the loop and proceeds to Stage 7. Therefore the loop executes at most $R_{\max}$ iterations.

\emph{Claim 2: The option set is bounded and non-increasing after cutoff.} Stage 2 initializes $|O| \le K_{\max}$. Stage 4 can only reduce or maintain $|O|$ (via merging, discarding dominated options). New hypotheses are admitted only before round $R_{\max} - 1$. After the cutoff, $|O|$ is non-increasing. Stage 4 selects at most $M$ finalists, so $|F| \le M$.

\emph{Claim 3: Stage 7 produces exactly one winner.} The fallback cascade---outranking, minimax regret, robust satisficing, Integrator selection---is a total ordering over finalists. At each level, if no unique winner emerges, the next level is applied. Integrator selection from the top-2 is guaranteed to produce exactly one winner.

\emph{Claim 4: Recursive sessions are bounded.} Each recursive spawn decrements the remaining depth by 1 and carves budget from the parent's remaining rounds. At $\text{depth} = D_{\max}$, no further spawns are permitted. The tree-wide round ceiling (default: 50) provides an additional hard bound. Therefore the total number of sessions in the recursion tree is at most $(1 + B_{\max})^{D_{\max}}$ where $B_{\max}$ is the maximum children per level.

Combining Claims 1--4: the main loop terminates in $R_{\max}$ rounds, each round involves bounded computation over $O(K_{\max})$ options and $n$ delegates, the fallback is deterministic, and recursion is depth-bounded. Therefore \dcicf{} terminates in finite time with a structured decision packet. \qed
\end{proof}

\section{Interaction Move Schema}
\label{app:schema}

Each \dci{} interaction move conforms to the following schema:

\begin{lstlisting}[style=protocol,caption={Interaction move structure.},label={lst:move}]
{
  "move_id":     "mv-042",
  "session_id":  "DCI-S-001",
  "round":       2,
  "phase":       "mutual_engagement",
  "actor":       "Challenger",
  "mode":        "critical",        // Layer 1: speech mode
  "act":         "challenge",       // Layer 2: interaction act
  "intent":      "test assumption", // Layer 3: specific purpose
  "target":      "contribution:mv-031",
  "content":     "This proposal assumes delegates can
                  self-regulate without coordination
                  pressure. What prevents divergence?",
  "confidence":  0.78,
  "move_force":  "hard",
  "meta_level":  false
}
\end{lstlisting}

Valid act types: \act{frame}, \act{propose}, \act{clarify}, \act{ask}, \act{challenge}, \act{extend}, \act{reframe}, \act{bridge}, \act{synthesize}, \act{ground}, \act{update}, \act{recommend}, \act{spawn}, \act{recall}.

Valid speech modes: \emph{exploratory}, \emph{analytical}, \emph{critical}, \emph{integrative}, \emph{reflective}, \emph{decisional}.

\section{Full Experimental Configuration}
\label{app:config}

Table~\ref{tab:config} lists the complete parameter configuration used in all experiments.

\begin{table}[h]
\centering
\caption{Experimental configuration parameters.}
\label{tab:config}
\small
\begin{tabular}{@{}llr@{}}
\toprule
\textbf{Parameter} & \textbf{Description} & \textbf{Value} \\
\midrule
\multicolumn{3}{@{}l}{\emph{\dcicf{} Parameters}} \\
\textit{max\_rounds}      & Maximum deliberation rounds       & 2 \\
\textit{max\_options}      & Maximum candidate options          & 5 \\
\textit{finalist\_count}   & Finalists for scoring              & 3 \\
\textit{convergence\_margin} & Score dominance threshold $\epsilon$ & 0.15 \\
\textit{fallback\_rule}    & Forced-decision method             & outranking \\
\textit{max\_depth}        & Recursive session depth limit      & 2 \\
\textit{tree\_ceiling}     & Total rounds across session tree   & 50 \\
\midrule
\multicolumn{3}{@{}l}{\emph{Model Configuration}} \\
Delegate model             & Delegate LLM                       & Gemini 2.5 Flash \\
API endpoint               & Google GenAI (OpenAI-compat.)      & generativelanguage.googleapis.com \\
Judge model                & LLM-as-judge                       & Gemini 3 Flash Preview \\
\textit{reasoning\_effort} & Thinking token control             & ``none'' (disabled) \\
\textit{max\_tokens}       & Response length bound              & 16384 \\
Rate limiting              & Minimum inter-request delay        & 4 seconds \\
Temperature (delegates)    & Sampling temperature               & 0.7 \\
Temperature (judge)        & Sampling temperature               & 0.2 \\
\midrule
\multicolumn{3}{@{}l}{\emph{Evaluation}} \\
Tasks per domain           & Evaluation set size                & 5--10 \\
Quality rubric scale       & LLM-as-judge scoring               & 1--10 \\
Significance test          & Statistical test                   & Wilcoxon \\
Multiple comparison        & Correction method                  & Holm-\v{S}id\'{a}k \\
Confidence intervals       & Bootstrap method                   & 95\% CI \\
\bottomrule
\end{tabular}
\end{table}

\end{document}